\documentclass[a4paper,UKenglish,cleveref, autoref, thm-restate]{lipics-v2021}

\title{Dynamic Blocked Clause Elimination for Projected Model Counting} 

\author{Jean-Marie Lagniez}{Univ. Artois, CNRS, CRIL, France \and \url{https://www.cril.univ-artois.fr/~lagniez/}}{lagniez@cril.fr}{https://orcid.org/0000-0002-6557-4115}{}

\author{Pierre Marquis}{Univ. Artois, CNRS, CRIL, IUF, France \and \url{https://www.cril.univ-artois.fr/~marquis/}}{marquis@cril.fr}{https://orcid.org/0000-0002-7979-6608}{}

\author{Armin {Biere}}{University Freiburg, Germany \and \url{https://cca.informatik.uni-freiburg.de/biere/}}{biere@cs.uni-freiburg.de}{https://orcid.org/0000-0001-7170-9242}{}

\authorrunning{J.-M. Lagniez, P. Marquis, A. Biere}

\Copyright{Jean-Marie Lagniez, Pierre Marquis, Armin Biere}



\category{} 

\relatedversion{} 

\acknowledgements{This work has benefited from the support of the AI Chair EXPEKCTATION (ANR-19-CHIA-0005-01) of the French National Research Agency.}

\supplement{}
\supplementdetails[subcategory={}, cite={}, swhid={}]{Software}{https://zenodo.org/records/11854002}

\nolinenumbers 

\EventEditors{John Q. Open and Joan R. Access}
\EventNoEds{2}
\EventLongTitle{42nd Conference on Very Important Topics (CVIT 2016)}
\EventShortTitle{CVIT 2016}
\EventAcronym{CVIT}
\EventYear{2016}
\EventDate{December 24--27, 2016}
\EventLocation{Little Whinging, United Kingdom}
\EventLogo{}
\SeriesVolume{42}
\ArticleNo{23}

\def\cnf{{\tt CNF}}
\def\bcp{{\tt BCP}}
\def\bce{{\tt BCE}}

\newcommand{\norm}[1]{\|#1\|}

\usepackage[vlined]{algorithm2e}
\usepackage{amssymb}

\usepackage{cancel}

\newcommand\nmodels{\mathbin{\cancel{\models}}}
\newcommand\nequiv{\mathbin{\cancel{\equiv}}}

\begin{document}

\maketitle

\begin{abstract}
In this paper, we explore the application of blocked clause elimination for projected model counting.  This is the problem 
of determining the number of models $\norm{\exists X . \Sigma}$ of a propositional formula $\Sigma$ after eliminating a given
set $X$ of variables existentially. 
Although blocked clause elimination is a well-known technique for SAT solving, its direct application to model counting is 
challenging as in general it changes the number of models.
However, we demonstrate, by focusing on projected variables during the blocked clause search, that blocked 
clause elimination can be leveraged while preserving the correct model count. 
To take advantage of blocked clause elimination in an efficient way during model counting, a novel data structure and associated algorithms are introduced. 
Our proposed approach is implemented in the model counter {\tt d4}. Our experiments demonstrate the computational 
benefits of our new method of blocked clause elimination for projected model counting.

\end{abstract}

\section{Introduction}
\label{sec:introduction}
Propositional model counting consists determines the number of models of a propositional formula $\Sigma$, 
typically represented in conjunctive normal form (\cnf). 
Many applications however require  a projected variant focusing on a specific set $X$ of variables of interest:
given a propositional formula $\Sigma$ and a set $X$ of propositional variables to be forgotten, the projected model 
counting problem consists in computing the number of interpretations over the variables occurring in $\Sigma$ but not in $X$, which 
coincide on $X$ with a model of $\Sigma$. 
In other words, the goal is to count the number of models of the quantified Boolean formula $\exists X. \Sigma$ over its variables 
(i.e., those present in $\Sigma$ but absent in $X$).

The projected model counting problem is significant for various application domains of artificial intelligence (AI).
For instance, in planning scenarios, it helps to evaluate the robustness of a plan by determining the number of initial states from 
which the plan execution leads to a goal state \cite{AzizCMS15}. 
Additionally, its applicability extends beyond AI to formal verification problems \cite{KlebanovMM13} and database operations \cite{AbiteboulHV95}.
As a generalization of the standard model counting problem {\tt \#SAT}, for the special case $X = \emptyset$, the projected 
model counting problem is at least as complex as {\tt \#SAT} ({\tt \#P-hard}). 
However, the possibility to eliminate some variables actually also might simplify the problem, such as when all variables in 
$\Sigma$ belong to $X$, reducing the problem to simply determining the satisfiability of $\Sigma$. 
However, in practice projected model counting often turns out to be more challenging than the standard model 
counting problem. 
This can be explained by the additional constraints imposed on the branching heuristic, i.e., in which order variables can be used as decisions, making the problem inherently harder. 
This is also reflected for instance in upper known bounds in the literature~\cite{FichteHMTW23}
on the time complexity of model counting even formulas with fixed treewidth $k$, i.e., $\mathcal{O}(2^k)$ for standard
and $\mathcal{O}(2^{2^k})$ projected model counting.

One way to speed up model counting is to employ preprocessing which simplifies the formula before tackling the 
model counting task. 
Preprocessing methods have shown to be effective across various automated reasoning tasks,
notably in {\tt SAT} solving and {\tt QBF} solving \cite{BiereJK21}. 
Among these preprocessing techniques, blocked-clause elimination (\bce) \cite{JarvisaloBH10} significantly improves solver 
performance by emulating several other, more complex preprocessing techniques \cite{JarvisaloBH12}.
Blocked clauses, initially introduced by Kullmann \cite{Kullmann99} as a generalization of extended resolution clauses, 
are pivotal in propositional preprocessing techniques. 
In essence, a clause $\alpha$ is deemed blocked within a \cnf\ formula $\Sigma$ if it includes a literal $\ell$ for which all 
conceivable resolvents of $\alpha$ over $\ell$ yield tautologies.
Removal of blocked clauses can significantly enhance the performance of {\tt SAT} solvers \cite{JarvisaloBH10}.
Furthermore, generalized forms of \bce\ have demonstrated remarkable performance improvements in solving problems beyond {\tt NP}, 
such as {\tt QBF} \cite{HeuleJLSB15}, {\tt DQBF} \cite{WimmerGNSB15} and even first-order theorem proving~\cite{DBLP:conf/lpar/Kiesl0STB17}.

However, while several preprocessing techniques used for {\tt SAT} solving can be adapted to improve
model  counting \cite{LagniezM17,LagniezLM20}, others, such as the blocked clause elimination technique, are 
unsuitable due to their inability to preserve the number of models. 
In this paper, we address this challenge by delineating conditions under which the use of \bce\ is correct in projected model 
counting. 
Specifically, we demonstrate that focusing on projected variables during the blocked clause search is correct,
i.e., gives the same projected count. 
The rationale behind this lies in the fact that when concentrating on sub-formulas containing only projected variables, the 
requirement boils down to ensure satisfiability. 
Consequently, clauses blocked on projected variables can safely be removed.

When used for model counting, simplification techniques are typically applied up-front during preprocessing
and even though modern {\tt SAT} solvers make heavily use of
interleaving formula simplification with {\tt CDCL} search, also called \emph{inprocesssing}~\cite{DBLP:conf/cade/JarvisaloHB12},
this form of simplification is currently performed only at the root-level (decision level zero).
In this paper we go beyond root-level simplification and propose to
dynamically apply the blocked clause elimination technique dynamically during search
at every decision level in the form of \emph{dynamic blocked clause elimination}.
In this sense our approach is similar to look-ahead solving~\cite{DBLP:series/faia/HeuleM21}, which use simplification techniques during search, i.e., probing techniques, at every decision level.

To accomplish this, we introduce novel data structures and associated algorithms tailored for dynamic inprocessing. 
Our method efficiently identifies clauses eligible for elimination by employing a mechanism akin to watched literals. 
Importantly, this methodology is not tied to a specific model counter; it seamlessly integrates into any state-of-the-art 
model counter.

To assess the efficiency of our approach, we conducted experiments using the model counter \texttt{d4} \cite{LagniezM17a},
modified to integrate our newly developed data structures and algorithms for projected model counting. 
We evaluated the performance of this new version of {\tt d4} across various benchmarks from recent model counting 
competitions (available at \url{https://mccompetition.org/}).
Our experimental results underscore the computational advantages of employing blocked clause elimination for 
projected model counting. 
For certain benchmarks, the adoption of \bce\ dynamic inprocessing led to a substantial reduction in computation time, 
with time savings of up to one order of magnitude compared to the baseline version of \texttt{d4}.
To ensure that the improvements are indeed attributable to the use of \bce\ inprocessing, we also examined a version of 
\texttt{d4} that implements \bce\ dynamically during preprocessing only. 

Interestingly, our findings indicate that employing \bce\ in preprocessing had no discernible impact on the effectiveness of the 
model counter {\tt d4}, underscoring the possibility to take advantage of \bce\ eagerly during the model counting process.

The remainder of the paper is structured as follows. 
The next section provides formal preliminaries. 
Following this, we delve into theoretical insights and give implementation details on
how to perform
\bce~dynamically during search.
Then, we outline the experimental protocol adopted for our empirical evaluations, along with the corresponding results.
Finally, we conclude the paper, offering insights into potential avenues for future research. 
The source code and benchmarks utilized in our experiments are provided as supplementary materials.

\section{Preliminaries}
\label{sec:preliminaries}

Let $\mathcal{L}$ be a propositional language built up from a finite set of propositional variables $\mathcal{P}$ and the 
standard logical connectives. 
The symbols $\bot$ and $\top$ represent the Boolean constants for falsehood and truth, respectively. 
A {\em literal} $\ell$ is either a propositional variable (e.g., $x$) or its negation ($\neg x$). 
For a literal $\ell$ defined over variable $x$, its {\em complementary literal} $\overline{\ell}$ is defined as 
$\overline{\ell} = \neg x$ if $\ell = x$, and $\overline{\ell} = x$ if $\ell = \neg x$, with $\mathit{Var}(\ell) = x$ denoting 
the variable of $\ell$. 
A {\em term} is a conjunction of literals. A {\em clause} is a disjunction of literals. Terms and clauses
are also interpreted as their sets of literals whenever convenient. 

A clause is a {\em tautology} if it contains $\top$, or both $x$ and $\neg x$ for some variable $x$.
A \cnf\ formula $\Sigma$ is a conjunction of clauses, also viewed as set of clauses when needed. 
The set of propositional variables occurring in $\Sigma$
is denoted
$\textit{Var}(\Sigma)$.
If a variable $x \in X$ does not belong to $\textit{Var}(\Sigma)$, then $x$ is said to be \emph{free} in $\Sigma$.
Each clause is associated with a unique identifier represented as an integer. 
A clause $\alpha_i$ of a \cnf\ formula $\Sigma$ can be accessed using its identifier through square bracket notation, 
denoted as $\Sigma[i]$. 
Thus $\alpha_i$ is also noted $\Sigma[i]$.
We denote by $S_{\ell}(\Sigma)$ the set of clauses of $\Sigma$ that contain literal $\ell$. 
When no ambiguity about $\Sigma$ is possible, we simply use the shorthand notation $S_{\ell}$ instead of explicitly 
writing $S_{\ell}(\Sigma)$.

\begin{example}
\label{ex:running}
Consider the \cnf{} formula $\Sigma = \{\alpha_1, \alpha_2, \ldots, \alpha_{11}\}$ with
\begin{center}
\begin{tabular}{llll}
$1: x_1 \vee x_2$ &  
$2: \neg x_2 \vee x_3$ & 
$3: \neg x_1 \vee \neg x_2 \vee \neg y_1$ &
$4: x_1 \vee \neg x_3 \vee y_1$ \\
$5: x_2 \vee \neg x_3 \vee y_2$ & 
$6: x_1 \vee \neg x_3 \vee \neg y_2$ &
$7: y_3 \vee x_2$ & 
$8: \neg y_3 \vee \neg x_2 \vee \neg x_3$ \\
$9: \neg y_3 \vee x_1$ & 
$10: \neg y_3 \vee \neg y_2 \vee x_3$ & 
$11: y_3 \vee y_2 \vee x_2$
\end{tabular}
\end{center}
$\textit{Var}(\Sigma) = \{x_1, x_2, x_3, y_1, y_2, y_3\}$, 
$S_{x_1}(\Sigma) = \{\alpha_1, \alpha_4, \alpha_6, \alpha_9\}$, and $\Sigma[2] = \alpha_2 = \neg x_2 \vee x_3$.
\end{example}

An {\em interpretation} (or world) over $\mathcal{P}$, denoted by $\omega$, is a mapping from $\mathcal{P}$ to $\{0, 1\}$. 
Interpretations $\omega$ are often represented by sets of literals (one per variable in $\mathcal{P}$), of exactly those literals set to $1$ by $\omega$.
The collection of all interpretations is denoted by $\mathcal{W}$. 
An interpretation $\omega$ is a {\em model} of a formula $\Sigma \in \mathcal{L}$ if and only if it satisfies the 
formula in accordance with the usual truth-functional interpretation. 
The set of models of the formula $\Sigma$ is denoted by $\text{mod}(\Sigma)$, defined as 
$\{\omega \in \mathcal{W} \mid \omega \models \Sigma\}$. 
The symbol $\models$ denotes logical entailment, while $\equiv$ denotes logical equivalence. For any formulas $\Sigma, \Psi \in \mathcal{L}$, we have 
$\Sigma \models \Psi$ if and only if $\text{mod}(\Sigma) \subseteq \text{mod}(\Psi)$ and $\Sigma \equiv \Psi$ if and only 
if $\text{mod} (\Sigma) = \text{mod}(\Psi)$. 
The notation $\norm{\Sigma}$ indicates the number of models of $\Sigma$ over $\mathit{Var}(\Sigma)$.

\begin{example}[Example \ref{ex:running} cont'd] 
$\norm{\Sigma} = 9$ and the models of $\Sigma$ are:

\begin{tabular}{lll}
$\{\neg x_1,x_2,x_3,y_1,\neg y_2,\neg y_3\}$ & 
$\{x_1,\neg x_2,x_3,y_1,y_2,y_3\}$ & 
$\{x_1, \neg x_2,x_3, \neg y_1,y_2,y_3\}$ \\
$\{x_1,\neg x_2,\neg x_3,\neg y_1,y_2,y_3\}$ &
$\{x_1,\neg x_2,\neg x_3,y_1,\neg y_2,y_3\}$ & 
$\{x_1,\neg x_2,\neg x_3,y_1,y_2,y_3\}$ \\
$\{x_1,\neg x_2,\neg x_3,\neg y_1,\neg y_2,y_3\}$ & 
$\{x_1,x_2,x_3,\neg y_1,y_2, \neg y_3\}$ &
$\{x_1,x_2,x_3,\neg y_1,\neg y_2, \neg y_3\}$ \\
\end{tabular}
\end{example}

For a formula $\Sigma \in \mathcal{L}$ and a subset $X \subseteq \mathcal{P}$, $\exists X . \Sigma$ represents, up to logical 
equivalence, the most general consequence of $\Sigma$ that is independent of the variables in $X$ (see for instance
\cite{LangLM03} for details). We note $\textit{Var}(\exists X . \Sigma) = \textit{Var}(\Sigma) \setminus X$.
%
%
%

\begin{example}[Example \ref{ex:running} cont'd] 
Let $X = \{y_1, y_2, y_3\}$. We have $\norm{\exists X.\Sigma} = 4$ and the models of $\exists X . \Sigma$ over $Var(\exists X . 
\Sigma)$ are $\{\{\neg x_1, x_2, x_3\}, \{x_1, \neg x_2, x_3\},\{x_1, \neg x_2, \neg x_3\},\{x_1, x_2, x_3\}\}$.
\end{example}

The {\em conditioning} of a \cnf\ formula $\Sigma$ by a consistent term $\gamma$ results in the formula denoted by
$\Sigma_{\mid \gamma}$, where
$\Sigma_{\mid \gamma}$ is obtained from $\Sigma$ by removing each clause from containing a literal of $\gamma$ and simplifying the remaining clauses, by removing from them complementary literals to those in $\gamma$. 
If, during the simplification, a clause becomes empty, then $\Sigma_{\mid \gamma}$ is unsatisfiable. 

The conditioning of $\Sigma$ on $\ell$ is equivalent to the formula $\exists \textit{Var}(\ell) . (\Sigma \wedge \ell)$. When $\ell$ is a unit clause of $\Sigma$,
$\Sigma_{\mid \ell}$ is satisfiable if and only if $\Sigma$ is satisfiable.
%
\emph{Boolean Constraint Propagation} (\bcp) \cite{MoskewiczMZZM01} is the algorithm that, given a \cnf\ formula $\Sigma$, 
returns a \cnf\ formula closed under unit propagation, i.e., that does not contain any unit clauses. The resulting formula is obtained by repeating the unit propagation of a unit clause of $\Sigma$ in the formula $\Sigma$ while such a unit clause exists. 
The identifiers assigned to clauses in $\Sigma$ remain unaltered by \bcp. Consequently, $\bcp(\Sigma)[i]$ will 
retrieve the clause $\alpha_i$ resulting from the application of \bcp\ on $\Sigma$, which could be $\bot$, $\top$, or a subset of 
$\alpha_i$.

\begin{example}[Example \ref{ex:running} cont'd] 
The formula $\bcp(\Sigma_{\mid \neg x_1}) = (x_2) \wedge (x_3) \wedge (y_1) \wedge (\neg y_2) \wedge (\neg y_3)$ is the result of conditioning $\Sigma$ with the literal $\neg x_1$ and applying \bcp\ to $\Sigma_{\mid \neg x_1}$.
$\bcp(\Sigma_{\mid \neg x_1})[1] = \top$ and $\bcp(\Sigma_{\mid \neg x_1})[4] = y_1$.
\end{example}

The resolution rule asserts that given two clauses $\alpha_1 = \{\ell, a_1, \ldots, a_n\}$ and 
$\alpha_2 = \{\overline{\ell}, b_1, \ldots, b_m\}$, the resulting clause $\alpha = \{a_1, \ldots, a_n, b_1, \ldots, b_m\}$, 
is the \textit{resolvent} of $\alpha_1$ and $\alpha_2$ on the literal $\ell$. 
This operation is denoted as $\alpha = \alpha_1 \oplus \alpha_2$. 
This concept extends naturally to sets of clauses: for two sets $S_\ell$ and $S_{\overline{\ell}}$ containing clauses that all 
involve $\ell$ and $\overline{\ell}$, respectively, we define $S_\ell \oplus S_{\overline{\ell}} = \{\alpha_1 \oplus \alpha_2 | \alpha_1 \in S_\ell, 
\alpha_2 \in S_{\overline{\ell}}, \text{ and } \alpha_1 \oplus \alpha_2 \text{ is not a tautology}\}$.

\begin{example}[Example \ref{ex:running} cont'd] 
Let $S_{\neg x_1} = \{(\neg x_1 \vee \neg x_2 \vee \neg y_1)\}$, 
$S_{y_3} = \{(y_3 \vee x_2), (y_3 \vee y_2 \vee x_2)\}$
and $S_{\neg y_3} = \{(\neg y_3 \vee \neg x_2 \vee \neg x_3), (\neg y_3 \vee x_1), (\neg y_3 \vee \neg y_2 \vee x_3)\}$. 
We have $S_{y_3} \oplus S_{\neg y_3} = \{(x_1 \vee x_2), (\neg y_2 \vee x_3 \vee x_2), (y_2 \vee x_2 \vee x_1)\}$
 and $\{(x_1 \vee x_2)\} \oplus S_{x_1} = \emptyset$.
\end{example}

The simplification technique known as {\em Blocked Clause Elimination} (\bce)~\cite{JarvisaloBH10,HeuleJLSB15}, targets the removal of specific clauses termed 
{\em blocked clauses} from \cnf\ formulas \cite{Kullmann99}. 
In the context of a \cnf\ formula $\Sigma$, a literal $\ell$ within a clause $\alpha$ is termed a {\em blocking literal}
if it blocks $\alpha$ with respect to $\Sigma$. 
This occurs when, for every clause $\alpha'$ in $\Sigma$ containing $\overline{\ell}$, the resulting resolvent 
$\alpha \oplus \alpha'$ on $\ell$ is a tautology. 
In essence, for a given \cnf\ and its clauses, a clause is considered blocked if it contains a literal that can effectively 
block it.
Applying \bce\ to $\Sigma$ leads to remove every clause containing a blocking literal and by repeating the process iteratively until no blocked literal exists.
\cite{JarvisaloBH10,HeuleJLSB15} illustrates that the outcome of \bce\ remains satisfiable  
equivalent regardless of the sequence in which blocked clauses are eliminated. 
More generally, blocked clause elimination converges to a unique fixed point for any \cnf\ formula, establishing the confluence of the method.

\begin{example}[Example \ref{ex:running} cont'd] 
\label{ex:diffCountBlocked}
Above we have shown that the clause $(x_1 \vee x_2)$ is blocked by $x_1$ and therefore can be eliminated. 
Following this, both $(\neg x_1 \vee \neg x_2 \vee \neg y_1)$ and $(x_1 \vee \neg x_3 \vee y_1)$ can be removed interchangeably, 
as they are respectively blocked by $y_1$ and $\neg y_1$.
Subsequently, $(\neg y_3 \vee \neg y_2 \vee x_3)$, blocked by $\neg y_2$, is eliminated, along with the two clauses, 
$(\neg y_3 \vee x_1)$ and $(x_1 \vee \neg x_3 \vee \neg y_2)$, both blocked by $x_1$.
Next, $(y_3 \vee x_2)$ is removed as it is blocked by $y_3$. Following this, both $(x_2 \vee \neg x_3 \vee y_2)$ and 
$(y_3 \vee y_2 \vee x_2)$, blocked by $y_2$, can be eliminated.
Then, $(\neg y_3 \vee \neg x_2 \vee \neg x_3)$ is removed because it is blocked by $y_3$, and finally, the last clause $
(\neg x_2 \vee x_3)$ is removed because it is blocked by both $\neg x_2$ and $x_3$.
Thus, $\bce(\Sigma) = \emptyset$.
\end{example}

As highlighted in \cite{Kullmann99}, the removal of any blocked clause ensures the preservation of unsatisfiability. 
However, as illustrated by the previous example, utilizing blocked clause elimination (\bce) on a \cnf\ formula $\Sigma$ 
does not ensure that the resulting formula $\bce(\Sigma)$ has the same number of models as $\Sigma$.
In the next section, we will delve into the specific conditions under which \bce\ can be effectively used for projected model counting.

\section{Blocked Clause Elimination for Projected Model Counting}
\label{sec:contrib}
Our goal is to use blocked clause elimination dynamically during search in projected model counting.
The primary challenge is to identify conditions under which such simplification is allowed.
Section \ref{subsec:contribTheo} provides novel theoretical insights permitting the removal of
blocked clauses and
Section \ref{subsec:contribImplem} introduces new algorithms to
efficiently identify them.

\subsection{Theoretical Insights}
\label{subsec:contribTheo}

As illustrated by Example \ref{ex:diffCountBlocked}, the \bce\ rule cannot be applied indiscriminately. 
When applied to the formula $\Sigma$ provided in Example \ref{ex:running}, the result is a tautological formula, indicating 
that $\norm{\bce(\Sigma)} = 1$ (since $Var(\bce(\Sigma)) = \emptyset$ this correspond to $2^6 = 64$ models over $Var(\Sigma)$), 
which differs from $\norm{\Sigma} = 9$. 
It is essential to note that blocked clause elimination guarantees the preservation of satisfiability but not necessarily 
equivalence or the number of models.
However, the picture changes when addressing the projected model counting problem. 
As we will demonstrate in Proposition \ref{prop:main}, it is feasible to eliminate clauses that are blocked on projected variables.
The rationale behind this lies in the fact that when focusing on sub-formulas containing only projected variables, the 
requirement is only to ensure satisfiability. 
Consequently, clauses blocked on projected variables can be removed in this case:

\begin{proposition}
\label{prop:main}
Let $\exists x . \Sigma$ be an existentially quantified \cnf\ formula. If a non-tautological clause $\alpha \in \Sigma$ is blocked 
by a literal $\ell \in \alpha$ with $Var(\ell) = x$, then $\exists x . \Sigma$ is logically equivalent to $\exists x . \Sigma'$, where 
$\Sigma' = \Sigma \setminus \{\alpha\}$.
\end{proposition}

\begin{proof}

To establish the logical equivalence $\exists x . \Sigma \equiv \exists x . \Sigma'$, we need to demonstrate both 
(1)~$\exists x . \Sigma \models \exists x . \Sigma'$ and (2)~$\exists x . \Sigma' \models \exists x . \Sigma$.
For condition (1) since $\Sigma \models \Sigma'$ it follows directly that $\exists x . \Sigma \models \exists x . \Sigma'$.
Now, let us demonstrate the second condition. 
We have to prove for any interpretation $\omega$ satisfying $\exists x . \Sigma'$, that $\omega$ also satisfies 
$\exists x . \Sigma$. 
Consider an interpretation $\omega$ satisfying $\exists x . \Sigma'$. 
This means that $\omega$ satisfies $(\Sigma'_{\mid x} \vee \Sigma'_{\mid \neg x})$. 
We need to address two scenarios depending on whether $\omega$ satisfies $\Sigma'_{\mid x}$ or $\Sigma'_{\mid \neg x}$.
If $\omega$ satisfies $\Sigma'_{\mid x}$, then $\Sigma'_{\mid x} \equiv \Sigma_{\mid x}$. 
Since $\Sigma_{\mid x}$ entails $\Sigma_{\mid x} \vee \Sigma_{\mid \neg x}$, we conclude that $\omega$ satisfies 
$\exists x . \Sigma$.
Let us consider the second scenario where $\omega$ satisfies $\Sigma'_{\mid \neg x}$ but not $\Sigma'_{\mid x}$ (the case when $\omega \models \Sigma'_{\mid x}$ has just been discussed).
First, both $\Sigma'_{\mid x}$ and $\Sigma'_{\mid \neg x}$ contain clauses from $\Sigma'$ that do not involve variable $x$. 
Therefore, if $\omega$ does not satisfy $\Sigma'_{x}$ but satisfies $\Sigma'_{\neg x}$, there must be a clause 
$\beta \in \Sigma'$ with $\neg x \in \beta$ and $\omega \nmodels \beta_{\mid x}$.
Now, let us demonstrate that $\omega$ satisfies $\Sigma_{\mid \neg x}$. 
Since $\Sigma_{\mid \neg x} \equiv (\Sigma' \wedge \alpha)_{\mid \neg x} \equiv \Sigma'_{\mid \neg x} \wedge \alpha_{\mid\neg x}$,
we only need to show that $\omega$ satisfies $\alpha_{\mid \neg x}$.
As $\alpha$ is blocked on $x$ in $\Sigma$, each resolvent between $\alpha$ and a clause of $\Sigma$ containing 
$\neg x$ is a tautology. 
Particularly, $\beta \oplus \alpha$ is a tautology, implying that there exists a literal 
$\exists y \in \beta$ such that $\neg y \in \alpha$ and $x \neq y$. 
Since we have established that $\omega \nmodels \beta_{\mid x}$, this implies that $\omega$ satisfies $\neg y$, hence $\omega$ 
satisfies $\alpha_{\mid \neg x}$. 
This demonstrates that $\omega$ satisfies $\Sigma'_{\neg x} \wedge \alpha_{\neg x}$, and consequently, $\omega$ satisfies 
$\Sigma_{\mid \neg x}$.
Using similar reasoning as before, we can show that $\omega$ satisfies $\exists x . \Sigma$.
Therefore, for any interpretation $\omega$ that satisfies $\exists x . \Sigma'$, it follows that $\omega$ satisfies $\exists x . 
\Sigma$, proving $\exists x . \Sigma' \models \exists x . \Sigma$.
\end{proof}

Proposition \ref{prop:main} only considers formulas with a single existentially quantified and thus projected variable. This 
can be extended to sets of variables:

\begin{corollary}
\label{cor:extendProp}
Let $\exists X . \Sigma$ be an existentially quantified \cnf\ formula. If a non-tautological clause $\alpha \in \Sigma$ 
is blocked by a literal $\ell \in \alpha$ such that $Var(\ell) \in X$, then $\exists X . \Sigma$ is logically equivalent to 
$\exists X . \Sigma'$, where $\Sigma' = \Sigma \setminus \{\alpha\}$.
\end{corollary}

\begin{proof}
The proof is straightforward.
Proposition \ref{prop:main} establishes $\exists x . \Sigma \equiv \exists x . \Sigma'$.
Therefore, we directly deduce that $\exists X \setminus \{x\}.(\exists x.\Sigma) \equiv 
\exists X \setminus \{x\}.(\exists x.\Sigma')$.
\end{proof}

Corollary \ref{cor:extendProp} demonstrates the potential of utilizing blocked clause elimination to enhance projected model 
counters. 
Our objective is not only to identify the set of blocked clauses in preprocessing but also to perform this operation during search
dynamically.
However, naive algorithms for blocked clause elimination are in the worst case at least quadratic in the size of the formula, which is clearly infeasible for dynamic blocked clause elimination.
In the following section, we capitalize on the observation that model counters typically follow the trace of DPLL solvers. 
To efficiently detect blocking literals and remove blocked clauses, a dedicated data structure along with associated algorithms are designed.

\subsection{Implementation Details}
\label{subsec:contribImplem}

To improve the efficiency of identifying clauses eligible for removal through the blocked clause elimination rule, 
we introduce the \texttt{BlockedClauseManager} object in this section. 
This specialized utility integrates efficient structures and algorithms crafted for this purpose, and is not exclusive to the 
projected model counter \texttt{d4}. 
It can be seamlessly employed in any state-of-the-art projected model counter.

To identify clauses eligible for elimination due to being blocked by a literal, we use a mechanism akin to the 
concept of watched literals. 
%
%
%
Given a formula $\exists X . \Sigma$, we aim to capture scenarios where a clause $\alpha$ cannot be eliminated via 
the blocked clause elimination rule, which occurs when there is no literal $\ell \in \alpha$ such that $\alpha$ is blocked on $x$,
and $\textit{Var}(\ell) \in X$. 
Specifically, a clause $\alpha$ is not blocked on a literal $\bar{\ell} \in \alpha$ if there exists another clause $\alpha'$ such that 
$\bar{\ell} \in \alpha'$ and $\alpha \oplus \alpha'$ is not a tautology.
Consequently, the invariant we adopt stipulates that for each literal $\ell \in \alpha$ such that $\textit{Var}(\ell) \in X$,
either $\ell$ is assigned or there must exist a clause $\alpha'$ where $\bar{\ell} \in \alpha'$, and $\alpha \oplus \alpha' \nequiv \top$.

\begin{example}[Example \ref{ex:running} cont'd] 
\label{ex:invariant}
When evaluating $\alpha_3$, it is not feasible to associate the literal $\neg y_1$ with a clause from $\Sigma$ 
without resulting in a tautology. 
Therefore, $\alpha_3$ can be safely removed from $\Sigma$. 
Conversely, when examining $\alpha_{11}$, it is feasible to associate the literal $y_3$ with clause $\alpha_9$ 
and the literal $y_2$ with clause $\alpha_6$, demonstrating that $\alpha_{11}$ cannot be eliminated from the formula 
using the blocked clause elimination rule.
\end{example}

Since blocked elimination can ignore (implied) learned clauses~\cite{DBLP:conf/cade/JarvisaloHB12}, the set 
$\{\alpha\} \oplus S_{\bar{\ell}}$, representing possible resolutions on a literal $\ell$ concerning a 
clause $\alpha \in \Sigma$, can be computed once at the outset. 
Consequently, when the watched clause to assess whether $\alpha$ is blocked on $\ell$ is deactivated, it suffices 
to consider clauses in $\{\alpha\} \oplus S_{\bar{\ell}}$ rather than re-evaluating each clause of $S_{\bar{\ell}}$ to
determine if the resolution rule yields a tautology.
The first data structure incorporated into our \texttt{BlockedClauseManager} is thus a set of triples 
$(\ell, \alpha, \{\alpha\} \oplus S_{\bar{\ell}})$, referred to as $\texttt{protectedTriple}$.

\begin{algorithm}[h]
\DontPrintSemicolon
\caption{\label{alg:initTriple}{\tt initProtectedTriple}}
\hrulefill

\textbf{Input:} $\exists X . \Sigma$ an existentially quantified \cnf\ formula.\\

\hrulefill

\nl $\mathtt{protectedTriple} \leftarrow \emptyset$\;

\nl\For{$x \in X$}
{
	\nl \For{$\ell \in \{x, \neg x\}$}
	{
	\nl	\For{$\alpha \in S_{\ell}(\Sigma)$}
		{
	\nl	 	$\mathtt{protectedTriple} \leftarrow \mathtt{protectedTriple} \cup 
	\{(\ell, \alpha, \{\alpha\} \oplus S_{\bar{\ell}}(\Sigma)\}$\;
		}
	}
}

\hrulefill
\end{algorithm}

The function \texttt{initProtectedTriple}, outlined in Algorithm \ref{alg:initTriple}, is designed for this purpose. 
When provided with the existentially quantified \cnf\ formula $\exists X . \Sigma$, it begins by enumerating all variables $x$ 
in $X$ (lines 2--5). 
Subsequently, it iterates through each possible triple $(\ell, \alpha, \{\alpha\} \oplus S_{\bar{\ell}})$ such that 
$\ell \in \{x, \neg x\}$, $\ell \in \alpha$, and $\alpha \in S_{\ell}$ (lines 3--5), adding them into 
$\texttt{protectedTriple}$ (line 5).
Moving forward, we will primarily work with clause identifiers rather than the clauses themselves. 
Therefore, when referring to a clause $\alpha$ in the following discussions, we are actually addressing its identifier. 
This applies similarly to sets of clauses; we will focus on the set of identifiers corresponding to the clauses within the set.

\begin{example}[Example \ref{ex:running} cont'd]
Upon invoking the function \texttt{initProtectedTriple} on the existentially quantified formula $\exists X. \Sigma$ provided 
in Example \ref{ex:running}, the set $\texttt{protectedTriple}$ contains the following triples:
$(y_1, 4, \{\})$, $(\neg y_1, 3, \{\})$, $(y_2, 5, \{6\})$, $(y_2, 11, \{6\})$, $(\neg y_2, 10, \{\})$,
 $(\neg y_2, 6, \{5,11\})$,  $(y_3, 7, \{9,10\})$, $(y_3, 11, \{9\})$, $(\neg y_3, 8, \{\})$, $(\neg y_3, 9, \{7,11\})$,
$(\neg y_3, 10, \{7\})$.
\end{example}

For each triple $(\ell, \alpha, \mathcal{C})$ in $\texttt{protectedTriple}$, we need to watch a clause 
from $\mathcal{C}$ to ensure that clause $\alpha$ is not blocked by $\ell$. 
To achieve this, we incorporate into \texttt{BlockedClauseManager} a map of watching lists, denoted as \texttt{watches}.
This structure associates each clause $\alpha \in \Sigma$ with a set of triples \texttt{watches[$\alpha$]} that are being 
watched by $\alpha$.

Algorithm \ref{alg:initWatch} presents the pseudo-code for the function \texttt{initWatchList}. 
Given an existentially quantified \cnf\ formula $\exists X . \Sigma$, this function initializes the \texttt{watches} structure 
and returns the indices of blocked clauses $U$, which are the clauses for which it is impossible to associate a sentinel. 
The function begins by initializing the set of blocked clauses as empty (line 2). 
Then, it initializes the map \texttt{watches} by associating an empty set with each clause of $\Sigma$ (lines 2--3). 
Next, it iterates over the triples in the \texttt{protectedTriple} set to associate a sentinel with each of them (lines 4--6).
For each triple $t = (\ell, \alpha, \mathcal{C})$, where $\mathcal{C}$ represents the set of non-tautological clauses, 
the algorithm checks whether $\mathcal{C}$ is empty. 
If it is, $\alpha$ is added to the set of blocked clauses (line 5). 
Otherwise, a clause $\alpha'$ from $\mathcal{C}$ is selected, and the triple $t$ is added to the watching list of $\alpha'$ 
(line 6).

\begin{algorithm}[h]
\DontPrintSemicolon
\caption{\label{alg:initWatch}{\tt initWatchList}}
\hrulefill

\textbf{Input:} $\exists X . \Sigma$ an existentially quantified \cnf\ formula.\\
\textbf{Output:} $B$ is a set of identifiers of clauses that are blocked.\\

\hrulefill

\nl $U \leftarrow \emptyset$\;
\nl Let $\mathtt{watches}$ an empty map\;
\nl \lFor{$\alpha \in \Sigma$ s.t. $Var(\alpha) \cap X \neq \emptyset$}
{
	$\mathtt{watches}[\alpha] = \{\}$
}

\nl\For{$t = (\ell, \alpha, \mathcal{C}) \in \mathtt{protectedTriple}$}
{
	\nl \lIf{$\mathcal{C} = \emptyset$}
	{
	$U \leftarrow U \cup \{\alpha\}$
	}
	\nl  \lElse{
	$\mathtt{watches}[\alpha'] \leftarrow \mathtt{watches}[\alpha'] \cup \{t\}$ with $\alpha' \in \mathcal{C}$
	}
}
\nl \Return U\;

\hrulefill
\end{algorithm}

\begin{example}[Example \ref{ex:running} cont'd]
Upon calling the function \texttt{initWatchList} on the existentially quantified formula $\exists X. \Sigma$ provided in Example 
\ref{ex:running}, the following represents a potential initialization of the \texttt{watched} structure:

\begin{tabular}{ll}
$\mathtt{watches}[6] = \{(y_2, 5, \{6\}), (y_2, 11, \{6\})\}$ & 
$\mathtt{watches}[5] = \{(\neg y_2, 6, \{5,11\})\}$ \\
$\mathtt{watches}[9] = \{(y_3, 7, \{9,10\}), (y_3, 11, \{9\})\}$ & 
$\mathtt{watches}[7] = \{(\neg y_3, 9, \{7,11\}), (\neg y_3, 10, \{7\})\}$ \\
\multicolumn{2}{l}{$\mathtt{watches}[3] = \mathtt{watches}[4] = \mathtt{watches}[8] = 
\mathtt{watches}[10] = \mathtt{watches}[11] = \emptyset$}
\end{tabular}
\end{example}

To finalize the initialization of the \texttt{BlockedClauseManager} object, we incorporate two arrays for maintaining records of assigned variables and satisfied clauses.
The first array, named \texttt{isAssignedVar}, associates each variable in $X$ with a Boolean value set to {\tt true} if the variable is assigned, 
and {\tt false} otherwise.
The second array, named \texttt{isActiveClause}, associates each clause of $\Sigma$ (identified by their identifier) with a Boolean variable 
set to {\tt true} if the clause is active, and {\tt false} otherwise.
The arrays \texttt{isAssignedVar} and \texttt{isActiveClause} are initialized with {\tt false} and {\tt true}, respectively, 
for all their elements.
We also need a stack $S$ of pairs, each consisting of variables and clauses. 
This stack is used to track the changes made to \texttt{isAssignedVar} and \texttt{isActiveClause} 
during each call of the function {\tt propagate}.

Algorithm \ref{alg:init} outlines all the necessary steps for the initialization process. 
It begins by initializing the two arrays (lines 1--2). 
Then, the set of triples is initialized by invoking the function \texttt{initProtectedTriple} on $\exists X . \Sigma$. 
Next, the \texttt{watches} structure is initialized by calling the function \texttt{initWatchList} on $\exists X . \Sigma$, 
and the set of blocked clauses is collected in $U$. 
We initialize $S$ as an empty stack of pairs, where each pair consists of a set of variables and a set of clauses.
Finally, the function \texttt{propagate} is called to gather all the initially blocked clauses. 
This function, described afterwards, takes a set of inactive clauses and a set of freshly assigned variables as input, and 
returns a set of clauses that are identified as blocked (further details will be provided later).

\begin{algorithm}[h]
\DontPrintSemicolon
\caption{\label{alg:init}{\tt init}}
\hrulefill

\textbf{Input:} $\exists X . \Sigma$ an existentially quantified \cnf\ formula.\\
\textbf{Output:} $B$ is a set of identifiers of clauses that are blocked.\\

\hrulefill

\nl Let $\mathtt{isAssigned}$ be an array s.t. $\mathtt{isAssigned}[x] = \mathtt{false}$ for each $x \in X$\;
\nl Let $\mathtt{isActiveClause}$ be an array s.t. $\mathtt{isActiveClause}[\alpha] = \mathtt{true}$ for each 
$\alpha \in \Sigma$\;
\nl {\tt initProtectedTriple}($\exists X . \Sigma$)\;
\nl $U \leftarrow \mathtt{initWatchList}(\exists X . \Sigma)$\;
\nl $S$ is an empty stack of pairs of the form (variables, clauses)\;
\nl \Return {\tt propagate($U$, $\emptyset$)}\;

\hrulefill
\end{algorithm}

Before delving into the specifics of how the function \texttt{propagate} operates, it is important to highlight that when
conditioning a formula by a literal $x$, without rendering it unsatisfiable, there is no need to consider clauses shortened 
by this assignment.
Consider a clause $\alpha \in \Sigma$ with $\bar{\ell} \in \alpha$. 
We aim to demonstrate that $\alpha \setminus \{\bar{\ell}\} \in \Sigma_{\mid \ell}$ cannot be blocked by any literal 
$\ell' \in \alpha \setminus\{\ell\}$ in $\Sigma_{\mid \ell}$. 
Given that $\alpha$ is not blocked in $\Sigma$, for every $\ell' \in \alpha \setminus \{\bar{\ell}\}$, there exists $\alpha' \in \Sigma$ 
such that $\bar{\ell'} \in \alpha'$ and $\alpha \oplus \alpha' \nequiv \top$.
Firstly, note that $\ell \notin \alpha'$; otherwise, $\alpha \oplus \alpha'$ would be a tautology. 
We then consider two cases based on whether $\bar{\ell}$ belongs to $\alpha'$. 
In the first case when $\bar{\ell} \in \alpha'$, we have $\alpha' \setminus \{\bar{\ell}\} \in \Sigma_{\mid \ell}$, and since 
the resolution between $\alpha \setminus \{\bar{\ell}\}$ and $\alpha' \setminus \{\bar{\ell}\}$ is not a tautology, it follows that
$\alpha \setminus \{\ell\}$ is not blocked on $\ell'$. 
In the second case where $\bar{\ell} \notin \alpha'$, $\alpha' \in \Sigma_{\mid \ell}$, and once more, $\alpha \setminus \{\ell\}$ is not 
blocked on $\ell'$ because $\alpha' \setminus \{\ell\} \oplus \alpha' \nequiv \top$.

\begin{example}[Example \ref{ex:invariant} cont'd] 
Let us examine the formula $\Sigma_{\mid \neg x_2}$. It is evident that $\alpha_{11}$ remains an unblocked because, with the literals 
present in the resulting clauses $y_3 \vee y_2$, we can still reference the same clauses from $\Sigma_{\mid \neg x_2}$ to maintain 
the invariant.
\end{example}

There are two scenarios where it becomes pertinent to evaluate whether a clause can be eliminated due to being blocked: when 
an active clause has been satisfied by a literal, or when an active clause has been blocked by a literal.
Thus, once certain clauses become inactive, that are clauses satisfied or blocked, it becomes imperative to update 
the \texttt{watches} structure accordingly.
The process for this update closely resembles the mechanism for updating watched literals in modern SAT solvers.
Specifically, for each newly inactive clause $\alpha$, we need to iterate through the list of triples from 
$\mathtt{watches[\alpha]}$ associated with $\alpha$. 
For each triple $t = (\ell, \alpha, \mathcal{C})$ of $\mathtt{watches[\alpha]}$, if $\ell$ is not assigned and $\alpha$ is active, 
we must search for another sentinel in $\mathcal{C}$ -- that is, an active clause.
The concept here is to ensure that each triple $t = (\ell, \alpha, \mathcal{C})$ is linked with a clause. 
Additionally, if $t$ is active, meaning $\alpha$ is active, it should be watched by an active clause. 
Otherwise, if $t$ is watched by an inactive clause $\alpha'$, we must ensure that when we reactivate $\alpha$, $\alpha'$ 
is also made active again.
This aspect is crucial as it guarantees the backtrack freeness of our structure.

Algorithm \ref{alg:propagate} outlines the pseudo-code for the function \texttt{propagate}, which fulfills the aforementioned 
requirements.
It takes as input a set of clauses identified by their identifiers, denoted as $U$, that have become inactive, and a set of 
freshly assigned variables $Y$.
The function returns the set $B$ of clauses detected as being blocked.
The algorithm begins by updating the two arrays \texttt{isAssignedVar} and \texttt{isActiveClause} to reflect the assignment 
of variables in $Y$ and the inactivation of clauses in $U$ (lines 1--2).
Next, the set of blocked clauses $B$ is initialized to be empty (line 3).
Then, for each inactive clause $\alpha$ stored in $U$, the \texttt{watches} map is updated (lines 4--18).
To accomplish this, an inactive clause $\alpha$ is selected and removed from $U$ (lines 5--6).
Subsequently, the set \texttt{tmpWatch} is initialized to be empty, and it is used to store triples that will be watched 
by $\alpha$, containing triples with assigned variables or inactive clauses (line 18).
Since $\alpha$ is now inactive, active triples associated with $\alpha$ need to be redistributed to other active clauses.

This operation is conducted in the for loop where each triple $t = (\ell, \alpha', \mathcal{C})$ from $\mathtt{watches}[\alpha]$
is considered (lines 8--17).
If $\ell$ is assigned or if $\alpha'$ is inactive (line 9), $\alpha$ can continue to watch $t$, and thus $t$ is added to 
\texttt{tmpWatch} (line 10).
However, if there exists an active clause $\alpha''$ in $\mathcal{C}$ (line 11), then triple $t$ is added to the watch list 
of $\alpha''$ (line 12).
Lastly, if it is impossible to associate $t$ with an active clause (lines 13--17), then clause $\alpha'$ is considered blocked, 
implying that $\alpha$ can continue to watch $\alpha'$ since they will both become active together upon backtracking (line 14).
Moreover, $\alpha'$ is added to $B$, $\alpha'$ is added to $U$ to handle $\alpha'$ later (line 16), 
and clause $\alpha'$ is marked as inactive (line 17).
Upon completing the update of the \texttt{watches} map, the assigned variables and inactivated clauses are pushed
onto stack $S$ (line 19). 
Finally, the set of clauses identified as blocked is returned at line 20.

\begin{algorithm}[h]
\DontPrintSemicolon
\caption{\label{alg:propagate}{\tt propagate}}
\hrulefill

\textbf{Input:} $U$ represents a set of clause identifiers corresponding to the newly inactive clauses, 
and $Y$ denotes the set of variables that have been newly assigned.\\
\textbf{Output:} $B$ is a set of identifiers of clauses that are blocked.\\

\hrulefill

\nl \lFor{$x \in Y$}{$\mathtt{isAssigned[y]} = \mathtt{true}$}
\nl \lFor{$\alpha \in U$}{$\mathtt{isActiveClause[\alpha]} = \mathtt{false}$}
\nl $B \leftarrow \emptyset$\;
\nl \While{$U \neq \emptyset$}
{
	\nl Let $\alpha \in U$\;
	\nl $U \leftarrow U \setminus \{\alpha\}$\;
	\nl $\mathtt{tmpWatch} = \emptyset$\;
	\nl \For{$t = (x, \alpha', \mathcal{C}) \in \mathtt{watches}[\alpha]$}
	{
	\nl \uIf{$\mathtt{isAssigned}[x]$ {\bf or not} $\mathtt{isActiveClause[\alpha']}$}
	 {\nl $\mathtt{tmpWatch} \leftarrow \mathtt{tmpWatch} \cup \{t\}$\;}
	\nl\uElseIf{$\exists \alpha'' \in \mathcal{C}$ s.t. $\mathtt{isActiveClause}[\alpha']$}
		{\nl $\mathtt{watches}[\alpha''] \leftarrow \mathtt{watches}[\alpha''] \cup \{t\}$}
	\nl \Else{
	\nl	$\mathtt{tmpWatch} \leftarrow \mathtt{tmpWatch} \cup \{t\}$\;
	\nl $B \leftarrow B \cup \{\alpha'\}$\;
	\nl $U \leftarrow U \cup \{\alpha'\}$\;
	\nl $\mathtt{isActiveClause}[\alpha'] \leftarrow \mathtt{false}$\;
		}
	}
	\nl $\mathtt{watches}[\alpha] \leftarrow \mathtt{tmpWatch}$\;
}
\nl push the couple $(Y, B \cup U)$ in $S$\;
\nl \Return $B$\;	
\hrulefill
\end{algorithm}

\begin{example}[Example \ref{ex:running} cont'd]
Consider the scenario where the literal $x_1$ is assigned to {\tt true}. In this case, clauses $1, 4, 6, 9$ become satisfied.
Invoking the function \texttt{propagate} with this information will result in the function returning $\{11\}$ as the set of 
detected blocked clauses. 
The various structures within our \texttt{BlockedClauseManager} object will be updated as follows:

\begin{tabular}{ll}
$\mathtt{watches}[7] = \{(\neg y_3, 9, \{7,11\}), (\neg y_3, 10, \{7\})\}$ & 
$\mathtt{watches}[5] = \{(\neg y_2, 6, \{5,11\}), (\neg y_2, 10, \{5\})\}$ \\
$\mathtt{watches}[9] = \{(y_3, 11, \{9\})\}$ & 
$\mathtt{watches}[6] = \{((y_2, 11, \{6\})\}$ \\
$\mathtt{watches}[10] = \{(y_2, 5, \{6,10\}), (y_3, 7, \{9,10\})\}$ \\
\multicolumn{2}{l}{$\mathtt{watches}[3] = \mathtt{watches}[4] = \mathtt{watches}[8] = \mathtt{watches}[11] = \emptyset$ \hfill $S = (\{x_1\}, \{1,4,6,9,11\})$}\\
\multicolumn{2}{l}{$\mathtt{isActiveClause} = [0,1,1,0,1,0,1,1,0,1,0]$  \hfill  $\mathtt{isAssignedVar} = [1,0,0]$} \\
\end{tabular}
\end{example}

As mentioned earlier, our structure is designed to be backtrack-free. 
Therefore, the only operation needed during backtracking is to retrieve from the stack $S$ the elements of the two arrays 
\texttt{isAssignedVar} and \texttt{isActiveClause} that require reinitialization. 
Algorithm \ref{alg:backtrack} outlines the steps involved in the backtracking process.

\begin{algorithm}[h]
\DontPrintSemicolon
\caption{\label{alg:backtrack}{\tt backtrack}}
\hrulefill

\nl $(X, \mathcal{C}) \leftarrow \mathtt{top}(S)$\;
\nl $\mathtt{pop}(S)$\;
\nl \lFor{$x \in X$}{$\mathtt{isAssignedVar}[x] = \mathtt{false}$}
\nl \lFor{$\alpha \in \mathcal{C}$}{$\mathtt{isActiveClause}[\alpha] = \mathtt{true}$}
\hrulefill
\end{algorithm}

To conclude this section, let us illustrate how {\tt BlockedClauseManager} is used within a {\tt DPLL}-style projected model counter,
such as the one employed in the model counter {\tt d4} \cite{LagniezM17a}. 
It is worth noting that our proposed approach can also be applied to other types of projected model counters, such as those discussed in
\cite{SharmaRSM19,LagniezM19,DudekPV20a,DudekPV20b,HecherTW20}.
Algorithm \ref{alg:count} outlines the {\tt count} function, which is invoked on $\exists X . \Sigma$, an existentially quantified \cnf\ formula, 
returning the number of models of $\exists X . \Sigma$ over $Var(\Sigma) \setminus X$. 
Specifically, this function creates a global variable named $bce$, which is a {\tt BlockedClauseManager} object (line 1), 
initializes it (line 2), removes the detected blocked clauses (line 3), and calls the recursive algorithm {\tt count\_main} on the 
simplified formula (line 4). 
It is important to note that while this function initializes the object necessary for enforcing \bce\ during model counting, 
the actual computation of the number of models is performed within the {\tt count\_main} function, which is described afterwards.

\begin{algorithm}[h]
\DontPrintSemicolon
\caption{\label{alg:count}{\tt count}}
\hrulefill

\textbf{Input:} $\exists X . \Sigma$ an existentially quantified \cnf\ formula.\\
\textbf{Output:} the number of models of $\exists X . \Sigma$ over $Var(\Sigma) \setminus X$\\

\hrulefill

\nl {\tt \bf global} $bce$ is an \texttt{BlockedClauseManager} object \;
\nl $B \leftarrow$ {\tt init}($bce$)\;
\nl $\Sigma \leftarrow \Sigma \setminus \{\Sigma[i] \text{ s.t. } i \in B\}$\;
\nl \Return {\tt count\_main}($\Sigma$)\;
\hrulefill
\end{algorithm}

Algorithm \ref{alg:count_} outlines the recursive function {\tt count\_main}, which serves as a pseudo-code representation of a {\tt DPLL}-style
projected model counter. 
This function operates on $\exists X . \Sigma$, an existentially quantified \cnf\ formula, and computes the number of models of 
$\exists X . \Sigma$ over $Var(\Sigma) \setminus X$. 
The blue portion of the algorithm, which differs from the baseline due to the inclusion of \bce\ management, will be discussed later.

The function begins by invoking \bcp\ on the input formula $\Sigma$ at line 1. 
For simplicity, we assume that \bcp\ returns a triple consisting of the set of unit literals $units$, the set of satisfied clauses $S$, 
and the simplified formula $\Sigma$ without clauses from $S$ and the unit literals from $units$. 
If the formula returned by \bcp\ contains an empty clause, indicating unsatisfiability, the function returns 0 (line 3). 
At line 5, the algorithm visits a cache to determine whether the current formula $\Sigma$ has been previously encountered during the search. 
The cache, which starts empty, stores pairs comprising a \cnf\ formula and its corresponding projected model count with respect to $X$. 
Whenever $\Sigma$ is found in the cache, instead of recalculating $\norm{\exists X . \Sigma}$ from scratch, the algorithm retrieves 
${\tt cache}(\Sigma)$ (line 7) to streamline the computation.
If the formula is satisfiable, {\tt connectedComponents} is called (line 8) on $\Sigma$ to partition it into a set of \cnf\ formulae
that are pairwise variable-disjoint. 
This procedure is a standard method employed in model counters. 
It identifies connected components of the primal graph of $\Sigma$ and returns a set $comps$ of \cnf\ formulae, ensuring that each pair 
of distinct formulae in $comps$ does not share any common variable.
The variable $cpt$, used to accumulate intermediate model counts, is initialized to $1$ (line 9). 
Then, the function iterates over the connected components $\Sigma'$ identified in $comps$ (lines 10--14) to count the number of models 
of each component. 
If the considered component $\Sigma'$ only contains variables from $X$, the model count accumulated in $cpt$ is multiplied by $1$ if $\Sigma'$ 
is satisfiable and $0$ otherwise. 
If $Var(\Sigma') \setminus X$ is not empty, a variable $v$ from this set is chosen, and the function {\tt count\_main} is recursively 
called on $\Sigma'$ where $v$ is assigned true and on $\Sigma'$ where $v$ is assigned false. 
The results returned by the two recursive calls are then summed up and multiplied by the variable $cpt$ (line 14). 
Before returning the accumulated model count in $cpt$ (line 17), the formula $\Sigma$ is added to the cache associated with the corresponding 
projected model count $cpt$ (line 15).

To integrate \bce\ into the search process (blue part), we first call the function {\tt propagate} on $S$ and $units$ to update the information managed by the $bce$ object and compute the set of blocked clauses $B$ (line 3). 
Then, at line 4, we eliminate the stored set of blocked clauses $B$.
To maintain consistency between the \bce\ manager information and the ongoing recursive call, the backtrack function must be executed before 
returning the calculated model count (lines 6 and 16).

\begin{algorithm}[h]
\DontPrintSemicolon
\caption{\label{alg:count_}{\tt count\_main}}
\hrulefill

\textbf{Input:} $\exists X . \Sigma$ an existentially quantified \cnf\ formula.\\
\textbf{Output:} the number of models of $\exists X . \Sigma$ over $Var(\Sigma) \setminus X$\\

\hrulefill

\nl $(units, S, \Sigma) \leftarrow {\tt bcp}(\Sigma)$\;
\nl \lIf{$\bot \in \Sigma$}{\Return 0}
\nl \textcolor{blue}{$B \leftarrow bce.{\tt propagate}(S, \{x|\ell \in units \text{ and } Var(\ell) = x\})$ }\;
\nl \textcolor{blue}{$\Sigma \leftarrow \Sigma \setminus \{\Sigma[i] \text{ s.t. } i \in B\}$}\;
\nl \If{${\tt cache}(\Sigma) \neq nil$}
{
\nl \textcolor{blue}{$bce.{\tt backtrack()}$}\;
\nl \Return ${\tt cache}(\Sigma)$
}
\nl $comps \leftarrow {\tt connectedComponents}(\Sigma)$\;
\nl $cpt \leftarrow 1$ \;

\nl \For{$\Sigma' \in comps$}
{
	\nl \lIf{$Var(\Sigma') \setminus X = \emptyset$}
	{
		$cpt \leftarrow cpt \times ({\tt SAT}(\Sigma') ? 1 : 0)$		
	}
	\nl	\Else{
		\nl Let $v \in Var(\Sigma') \setminus X$ \;
		\nl $cpt \leftarrow cpt \times ({\tt count\_main}(\Sigma' \wedge v) + {\tt count\_main}(\Sigma' \wedge \neg v))$\; }
	}

\nl ${\tt cache}(\Sigma) \leftarrow cpt$\;
\nl \textcolor{blue}{$bce.{\tt backtrack()}$}\;
\nl \Return $cpt$\;

\hrulefill
\end{algorithm}

\section{Experimental Evaluation}
\label{sec:xp}
Our aim was to empirically assess the advantages of employing blocked clause elimination in solving instances of the
projected model counting problem. 
For our experimentation, we used 500 CNF instances from the three recent model counting competitions 
(the 2021, 2022, and 2023 editions documented at \url{https://mccompetition.org/}). 
We excluded instances from the 2020 competition due to incompatibility with our software caused by changes in the input format.
The instances were categorized into three datasets: 200 from the 2021 competition, 200 from the 2022 competition, and 100 
from the 2023 competition. 
Notably, as the full set of 2023 instances was unavailable at the time of writing, we only included the 100 public instances 
provided by the organizers.

The projected model counter used for the evaluation was {\tt d4} \cite{LagniezM17a}. 
Our experiments were conducted on Intel Xeon E5-2643 processors running at 3.30 GHz with 32 GiB of RAM, operating on Linux
CentOS. 
Regarding the model counting competition, each instance was subject to a time-out of 3600 seconds and a memory limit of 32 GiB.
For each instance, we measured the computation times required by three different versions of {\tt d4} for counting the numbers 
of projected models. 
These versions include:
\begin{itemize}
\item {\tt d4}: This is the standard version of {\tt d4}, as given at \url{https://github.com/crillab/d4v2}.
\item {\tt d4+BCE$_p$}: This version of {\tt d4} incorporates blocked clause elimination performed once during a preprocessing 
phase.
\item {\tt d4+BCE$_i$}: In this version of {\tt d4}, blocked clause elimination is performed dynamically 
throughout the search achieved by the model counter.
\end{itemize}
For all the versions under consideration, a preprocessing step of 60 seconds was conducted. 
This preprocessing involves running {\tt BiPe} \cite{LagniezM23}, followed by the occurrence elimination and vivification
preprocessing for 10 iterations as described in \cite{LagniezM14} (we only replace the gate simplification with {\tt BiPe}).

\begin{table}[t]
\centering
\begin{tabular}{c|c|c|c|c}
& 2021 (200) & 2022 (200) & 2023 (100) & All (500)\\
\hline
{\tt d4} & 139 (56 MO) & 149 (24 MO) & 73 (9 MO) & 361 (89 MO) \\
{\tt d4+BCE$_p$} & 139 (56 MO) & 149 (25 MO) & 73 (9 MO) & 361 (90 MO)\\
{\tt d4+BCE$_i$} & {\bf 172} (23 MO) & {\bf 163} (8 MO) & {\bf 78} (4 MO) & {\bf 413} (35 MO) \\
\end{tabular}
\caption{\label{tab:global}The table shows the numbers of instances solved by different versions {\tt d4}
within a time limit of 3600 seconds and a memory limit of 32 GiB. The number of memory out (MO) are reported between
brackets.}

\end{table}

Table \ref{tab:global} presents the number of instances for which different versions of {\tt d4} terminated within the
specified time and memory constraints. 
The correctness of the extended versions of {\tt d4} was verified by comparing their returned model counts with those of the
baseline version. 
For all instances solved by the baseline version, the extended versions returned the same model counts.
The table clearly demonstrates that leveraging dynamical blocked clause elimination significantly improves the performance
of the model counter in practice. 
Furthermore, regardless the benchmark set considered, the version of {\tt d4} equipped with dynamic blocked clause
elimination systematically solved more instances than the two other versions. 
This indicates that the improvement is not limited to specific benchmark sets. 
Moreover, Table \ref{tab:global} shows that using blocked clause elimination solely during preprocessing phase 
did not lead to increase the number of instances solved. 
This demonstrates that for effective results, blocked clause elimination needs to be performed eagerly.

\begin{figure}[h]
     \centering
     \begin{subfigure}[t]{0.45\textwidth}
         \centering
		 \includegraphics[width=1.1\textwidth]{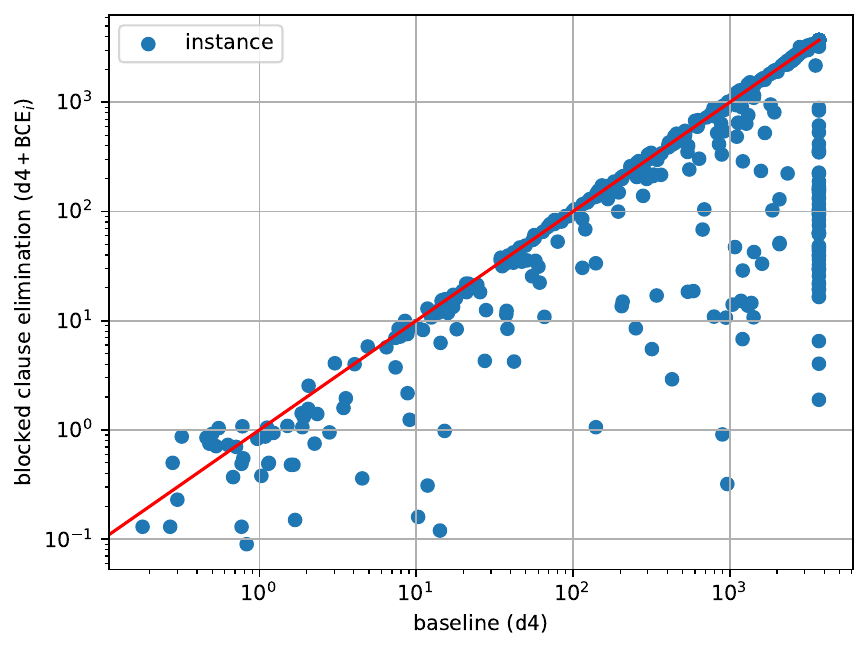}
		 \caption{\label{fig:scatter} Comparing the run times of {\tt d4} and {\tt d4+BCE$_i$}.}         
     \end{subfigure}
     \hfill
     \begin{subfigure}[t]{0.45\textwidth}
         \centering
         \includegraphics[width=1.1\textwidth]{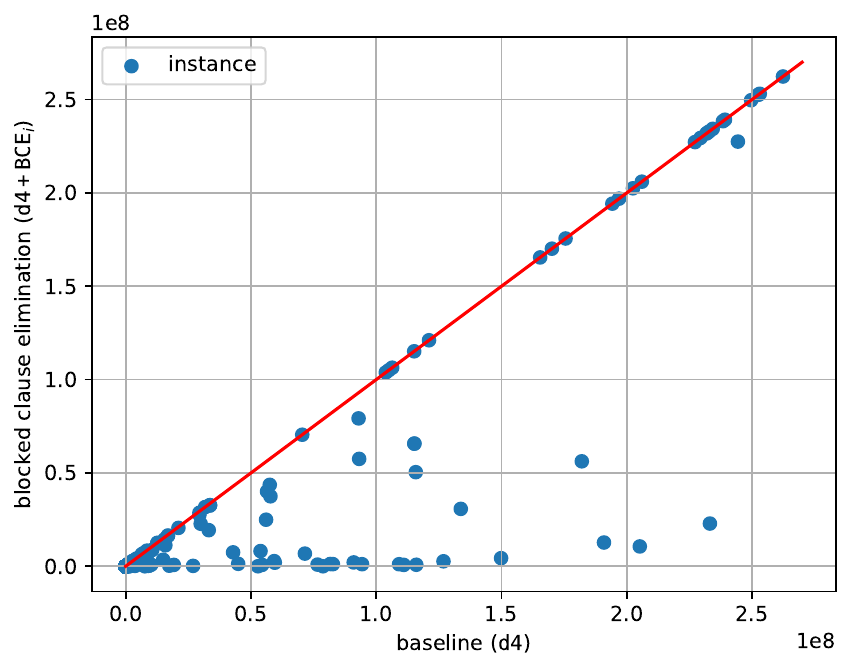}
     	 \caption{\label{fig:scatterDecision} Comparing the number of decisions made by {\tt d4} and {\tt d4+BCE$_i$} on the instances solved by both.}
     \end{subfigure}
     \hfill
     \begin{subfigure}[t]{0.45\textwidth}
         \centering
		\includegraphics[width=1.1\textwidth]{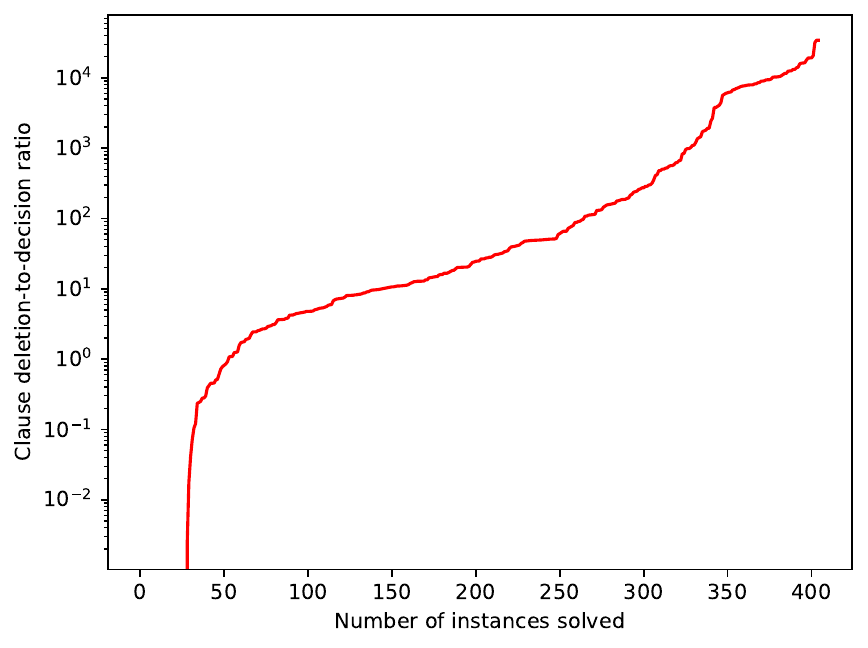}
		\caption{\label{fig:plotDelDec} Plot used to estimate the extent of clause deletion relative to the number of decisions made for instances solved by {\tt d4+BCE$_i$}.}
     \end{subfigure}
      \caption{\label{fig:xp}Experimental results}        
\end{figure}

Figure \ref{fig:scatter} presents a pairwise comparison between {\tt d4} and {\tt d4+BCE$_i$} on a scatter plot. 
%
%
Each data point represents an instance, with the x-axis indicating the time (in seconds) required to solve it using the
baseline version of {\tt d4}, and the y-axis representing the time for the enhanced version {\tt d4+BCE$_i$}. 
The experimental results unequivocally demonstrate that the version of {\tt d4} with dynamic blocked clause elimination 
generally outperforms the baseline version of {\tt d4}. 
Furthermore, the figure reveals instances where {\tt d4+BCE$_i$} achieves speeds one order of magnitude faster than the
baseline version {\tt d4}.

Focusing on instances solved by both approaches and exhibiting a solving time difference of more than five
seconds, the baseline version {\tt d4} beat {\tt d4+BCE$_i$} for 75 instances, achieving an average speedup of 4\%. 
This speedup is calculated as the ratio between the running times of the methods, with a peak improvement of 14\% and the third
quartile indicating a speedup of 7\%. 
We investigated the factors contributing to the greater efficiency of the baseline version compared to the one employing
dynamic blocked clause elimination, but we were unable to identify a definitive reason. 
It is hypothesized that the removal of clauses may negatively impact branching, potentially leading to a slightly larger search 
space explored by the model counter.
On the contrary, we discovered 150 instances for which {\tt d4+BCE$_i$} outperformed the baseline version of {\tt d4}. 
For them, {\tt d4+BCE$_i$} exhibits an average speed increase of 40 times compared to the baseline, with a peak 
improvement of up to 3000 times. 
The second quartile of the time distribution demonstrates a 50\% improvement, while the third quartile shows a 
remarkable 600\% enhancement.

Figure \ref{fig:scatterDecision} showcases the number of decisions made by {\tt d4} and {\tt d4+BCE$_i$} on instances 
solved by both methods. 
This visualization sheds light on the fact that the performance enhancement cannot be solely attributed to a reduction in 
memory usage, which might otherwise account for the observed decrease in memory consumption with {\tt d4+BCE$_i$}. 
It is widely recognized that the cache structure frequently serves as the primary memory bottleneck. 
Consequently, removing clauses reduces the size of cache entries, which typically results in decreased memory consumption and 
associated memory overhead.
Nevertheless, as depicted in Figure \ref{fig:scatterDecision}, the use of blocked clause dynamic elimination also 
results in a decrease in the number of decisions required by the model counter to complete its task. 
This underscores that the performance gain is not solely a consequence of an inadequate memory limit setting. 
Thus, even with a significant increase in the memory limit, employing dynamic blocked clause elimination proves highly 
advantageous in practice.
Specifically, for 71 instances, {\tt d4} required fewer decisions than {\tt d4+BCE$_i$}, with an average difference of 6460
decisions in favor of {\tt d4}. 
The second quartile exhibited a difference of 102 decisions, while the third quartile showed a difference of 306 decisions 
across these instances. 
Conversely, {\tt d4+BCE$_i$} required fewer decisions than {\tt d4} for 231 instances, with an average difference of 12,249,678 
decisions in favor of {\tt d4+BCE$_i$}. 
The second quartile exhibited a difference of 33,005 decisions, while the third quartile showed a difference of 1,602,452 
decisions across these instances.\\

Figure \ref{fig:plotDelDec} gives the proportion of clauses removed through dynamic blocked clause 
relative to the number of decisions made for instances solved by {\tt d4+BCE$_i$}. 
As observed in the plot, for approximately 300 instances, the average number of blocked clauses removed at each decision is at 
least 10. 
For about 100 of these instances, the average number of blocked clauses removed at each decision is at least 100. 
Additionally, for certain benchmarks, more than 1000 clauses where removed at each decision. 
While there is some variation in the extent of deletion across different steps, the plot clearly demonstrates that a 
substantial number of clauses are generally eliminated when employing dynamic blocked clause elimination.

\section{Conclusion and Perspectives}
\label{sec:conclu}
In conclusion, this paper has explored the utilization of the blocked clause elimination dynamically during projected model 
counting. 
Despite its widespread application in the satisfiability problem, the blocked clause elimination rule posed challenges for model 
counting due to its inability to maintain the number of models unchanged. 
However, through focused attention on projected variables during the search for blocked clauses, we have demonstrated the 
feasibility of leveraging this rule while preserving the correct model count.
To achieve this, we introduced a new data structure and corresponding algorithms tailored for leveraging blocked clause 
elimination dynamically during search.
This innovative machinery has been integrated into the projected model counter \texttt{d4}, enabling us to conduct comprehensive
experiments that showcase the computational benefits of our approach. 
Our results underscore the efficacy of leveraging the blocked clause elimination rule technique for projected model counting, 
opening avenues for further exploration and refinement in this domain.

Exploring extensions of blocked clause elimination (BCE) in the context of projected model counting
is interesting future work.
This particularly includes considering the elimination of resolution asymmetric tautologies~(RAT)~\cite{DBLP:conf/cade/JarvisaloHB12}, or even covered~\cite{DBLP:conf/lpar/HeuleJB10a,DBLP:conf/cade/BarnettCB20} or propagation redundant (PR)~\cite{DBLP:journals/jar/HeuleKB20} clauses.
These approaches hold the potential to uncover additional redundant clauses, that can be eliminated
and thus improve efficiency of projected model counting.
In addition, we envision the development of novel branching heuristics designed to prioritize the elimination of clauses that prevent 
removal of blocked clauses. 
These improved decision heuristics, could create more instances where clauses become blocked and thus eliminated,
again with the goal to improve solver efficiency.
Furthermore, we want to explore the applicability of blocked clause elimination to other reasoning tasks, particularly
to the weighted {\tt Max\#SAT} problem \cite{AudemardLMR22,AudemardLM22} or counting tree models of QBF formulas \cite{LagniezCPS24}.
%
%



\bibliography{biblio}

\end{document}